\documentclass[sigconf]{acmart}

\renewcommand\footnotetextcopyrightpermission[1]{} 
\usepackage{balance} 

\usepackage{graphicx}
\usepackage{hyperref}
\usepackage{amsmath,amsthm}
\usepackage{booktabs} 
\usepackage{xspace} 
\usepackage[inoutnumbered,linesnumbered,ruled,vlined]{algorithm2e} 
\usepackage{mathtools}
\usepackage{balance}
\usepackage{float}
\usepackage{enumitem}

\usepackage[skip=2pt]{caption}

\newcommand{\oea}{${(1 + 1)}$~EA\xspace}
\newcommand{\ocl}{${(1 , \lambda)}$~EA\xspace}

\newcommand{\opl}{${(1 + \lambda)}$~EA\xspace}
\newcommand{\saocl}{self-adjusting $(1,\lambda)$~EA\xspace}
\newcommand{\saopl}{$(1+\{F^{1/s}\lambda, \lambda/F \})$~EA\xspace}
\newcommand{\ga}{${(1 +(\lambda,\lambda))}$~GA\xspace}

\newcommand{\onemax}{\textsc{OneMax}\xspace}

\newcommand{\jump}{\textsc{Jump}\xspace}
\newcommand{\leadingones}{\textsc{LeadingOnes}\xspace}

\newcommand{\onemaxblocks}{\textsc{OneMaxBlocks}\xspace}

\newcommand{\cliff}{\textsc{Cliff}\xspace}
\newcommand{\lambdatwo}{\lambda_{\mathrm{inc}}\xspace}
\newcommand{\lambdaone}{\lambda_{\mathrm{safe}}\xspace}
\newcommand{\lambdafail}{\lambda_{\textrm{fail}}\xspace}
\newcommand{\pimp}{p_{x,\lambda}^+}

\newcommand{\ploss}{p_{x,\lambda}^-}
\newcommand{\pmin}{p^+_{\min}}
\newcommand{\pmax}{p^+_{\max}}

\newcommand{\Deltaloss}{\Delta_{x,\lambda}^-}
\newcommand{\Deltagain}{\Delta_{x,\lambda}^+}

\newcommand{\E}[1]{\text{E}\left(#1\right)}
\newcommand{\Prob}[1]{\mathrm{Pr}\left(#1\right)}
\newcommand{\prob}[1]{\Prob{#1}}

\newcommand{\N}{\mathbb{N}}

\newcommand{\zerobit}{0\nobreakdash-bit\xspace}

\newcommand{\zerobits}{0\nobreakdash-bits\xspace}
\newcommand{\onebits}{1\nobreakdash-bits\xspace}

\renewcommand{\epsilon}{\varepsilon}

 \acmConference[]{}{}{}

\allowdisplaybreaks[3]

\begin{document}

\title{Hard Problems are Easier for Success-based Parameter Control}

\author{Mario Alejandro Hevia Fajardo}
\affiliation{%
  \institution{Department of Computer Science}
  \city{University of Sheffield, Sheffield, UK}
}

\author{Dirk Sudholt}
\affiliation{%
  \institution{Chair of Algorithms for Intelligent Systems}
  \city{University of Passau, Passau, Germany}
}


\begin{abstract}
Recent works showed that simple success-based rules for self-adjusting parameters in evolutionary algorithms (EAs) can match or outperform the best fixed parameters on discrete problems.
Non-elitism in a \ocl combined with a self-adjusting offspring population size~$\lambda$ outperforms common EAs on the multimodal \cliff problem.
However, it was shown that this only holds if the success rate~$s$ that governs self-adjustment is small enough. Otherwise, even on \onemax, the \saocl stagnates on an easy slope, where frequent successes drive down the offspring population size.

We show that self-adjustment works as intended in the absence of easy slopes.
We define everywhere hard functions, for which successes are never easy to find and show that the self-adjusting \ocl is robust with respect to the choice of success rates~$s$. We give a general fitness-level upper bound on the number of evaluations and show that the expected number of generations is at most $O(d + \log(1/\pmin))$ where $d$ is the number of non-optimal fitness values and $\pmin$ is the smallest probability of finding an improvement from a non-optimal search point. We discuss implications for the everywhere hard function \textsc{LeadingOnes} and a new class \textsc{OneMaxBlocks} of everywhere hard functions with tunable difficulty. 
\end{abstract}


%
%


\maketitle

\pagestyle{plain}

\section{Introduction}

Evolutionary algorithms exhibit a number of important parameters such as the mutation rate, the population size and the selection pressure that have to be chosen carefully to obtain the best possible performance. There is wide empirical evidence, including a number of runtime analyses~\cite{DoerrSurvey2020}, showing that the performance of evolutionary algorithms can drastically depend on the choice of parameters. One way of choosing parameters is to try and adapt parameters during the course of a run. This is called parameter control, and it forms a key component in evolutionary algorithms for continuous domains. In the discrete domain, studying the benefit of parameter control from a theoretical perspective has become a rapidly emerging area. Many recent results have shown that parameter control mechanisms can compete with or even outperform the best static parameter settings~\cite{Badkobeh2014,Boettcher2010,Doerr2015,Doerr20201,DoerrSurvey2020,LehreD2020}. 


Even conceptually simple mechanisms have turned out to be surprisingly powerful. The idea behind \emph{self-adjusting} parameters is to adapt parameters according to information gathered throughout the run of an algorithm. 
One example is the information whether the current generation leads to an improvement of the current best fitness or not. The former is called a success. We describe the approach for self-adjusting the offspring population size~$\lambda$ using so-called \emph{success-based rules} that will be studied in this work. It is a variant of the famous one-fifth success rule~\cite{Rechenberg1973} from~\cite{KMH04} that was first studied in the discrete domain in~\cite{Doerr2018}. For an update strength $F>1$ and a success rate $s>0$, in a generation where no improvement in fitness is found, $\lambda$ is increased by a factor of $F^{1/s}$. In a successful generation, $\lambda$ is divided by a factor~$F$. If precisely one out of $s+1$ generations is successful, the value of $\lambda$ is maintained. The case $s=4$ is the one-fifth success rule~\cite{Rechenberg1973,KMH04}.



Studying parameter control mechanisms in the discrete domain is a rapidly emerging topic. The first runtime analysis by~\citet{Lassig2011} concerned self-adjusting the offspring population size in the \opl (using a simpler mechanism with hard-coded values $F=2$ and $s=1$) and adapting the number of islands in island models. 
\citet{Mambrini2015} adapted the migration interval in island models and showed that adaptation can reduce the communication effort beyond the best possible fixed parameter. \citet{Doerr2018} proposed the aforementioned self-adjusting mechanism in the \ga based on the one-fifth rule (using $s=4$) and proved that it optimises the well known benchmark function \onemax$(x) = \sum_{i=1}^n x_i$ in $O(n)$ expected evaluations, being the fastest known unbiased genetic algorithm on \onemax. 
The present authors (\citet{Hevia2020}) studied modifications to the self-adjusting mechanism in the \ga on \jump functions, showing that they can perform nearly as well as the \oea with the optimal mutation rate. \citet*{DoerrDK2018} presented a success-based choice of the mutation strength for an RLS variant, proving that it is very efficient for a generalisation of the \onemax problem to larger alphabets.
\citet*{Doerr2019opl} showed that a success-based parameter control mechanism is able to identify and track the optimal mutation rate in the (1+$\lambda$)~EA on \onemax, matching the performance of the best known fitness-dependent parameter~\cite{Badkobeh2014,LehreD2020}. \citet*{DoerrSAMut2021} proved that a success-based parameter control mechanism based on the one-fifth rule is able to achieve an asymptotically optimal runtime on \leadingones with the best performance obtained, when using a success rate $s=e-1$. \citet{CarstenSD2020} proposed a stagnation detection mechanism that raises the mutation rate when the algorithm is likely to have encountered a local optima. The mechanism can be added to any existing EA; when added to the \oea, the SD--\oea has the same asymptotic runtime on \jump as the optimal parameter setting. \citet{CarstenSD2021} further added the stagnation detection mechanism to RLS, obtaining a constant factor speed-up from the SD--\oea. 

All these results show that self-adjustment can be very effective, sometimes even beating optimal static parameters, and, crucially, \emph{without needing to know which static parameters are optimal}. We note that self-adjusting mechanisms often come with extra hyper-parameters, but these tend to be easier to choose and more robust.




Most theoretical analyses of self-adjusting mechanisms focus on \emph{elitist} EAs that always reject worsening moves. 
Elitism facilitates a theoretical analysis and many easy problems (such as \onemax) benefit from elitism. 
However, the performance improvements obtained through parameter control for elitist algorithm is fairly mild in many of the above examples.

In our previous work~\cite{Hevia2021arxiv,Hevia2021FOGA} we recently argued that larger speedups can be obtained by studying non-elitist evolutionary algorithms on harder problems. We 
studied a \emph{non-elitist} \ocl with self-adjusting offspring population size~$\lambda$ on the multimodal \cliff function~\cite{Jagerskupper2007a}. The function is defined similarly to \onemax, however evolutionary algorithms are forced to jump down a ``cliff'' in fitness, i.\,e.\ accept a significant fitness loss, or to jump across a huge fitness valley. This task is very difficult for several randomised search heuristics, including the (1+1)~EA~\cite{Paixao2016}, the Metropolis algorithm~\cite{LissovoiMultimodal2019} and the Strong Selection Weak Mutation algorithm (SSWM)~\cite{Paixao2016}. 
Comma selection in a \ocl can be more effective~\cite{Jagerskupper2007a}, however even with an optimal value of~$\lambda$, the expected optimisation time of the \ocl is a polynomial of large degree $\eta \approx 3.9767$, up to sub-polynomial factors~\cite{Hevia2021FOGA}.

We also showed~\cite{Hevia2021FOGA} that a self-adjusting \ocl, enhanced with a mechanism resetting~$\lambda$ if it grows too large, is able to optimise \cliff in $O(n)$ expected generations and $O(n \log n)$ expected evaluations. The algorithm has a good chance of jumping down the cliff after a reset of~$\lambda$, when the offspring population size is still small, and to then recover an offspring population size large enough to enable hill climbing to the global optimum.
This makes the self-adjusting \ocl the fastest known common evolutionary algorithm\footnote{Earlier work showed the same $O(n \log n)$ bound, but required components from artificial immune systems, such as ageing~\cite{Corus2020} and hypermutations~\cite{LissovoiMultimodal2019}.} on \cliff.

However, this success is only possible when the success rate~$s$ that governs the self-adaptation mechanism is 
appropriately small. 
In \cite{Hevia2021} (and refined results in~\cite{Hevia2021arxiv}) we showed that, even on the simple function \onemax, while the self-adjusting \ocl optimises \onemax in $O(n)$ expected generations and $O(n \log n)$ evaluations if $s < 1$ and $F>1$, the algorithm requires exponential time with high probability if $s \ge 18$ and $F\le1.5$.
The reason is that for large values of $s$, unsuccessful generations only increase the population size very slowly, by a factor of $F^{1/s}$ (which converges to~1 as $s$ grows), whereas successful generations decrease~$\lambda$ by a comparatively large factor of $F$. If the algorithm finds frequent improvements, the offspring population size is likely to decrease to very low values. Then the algorithm has a very low selection pressure and is likely to accept worsenings, namely if all offspring are worse than their parent. Such a situation occurs on the slope to the global optimum of \onemax, when the current search point still has a linear Hamming distance from the optimum and improvements are found easily. This behaviour is not limited to \onemax; it holds for other common benchmark functions that have easy slopes~\cite{Hevia2021arxiv}. 

In this work we show that the self-adjusting \ocl is robust with respect to the choice of the success rate if the fitness function is sufficiently hard. We define a class of \emph{everywhere hard} fitness functions, where for all search points the probability of finding an improvement is bounded by $n^{-\varepsilon}$, for a constant $\varepsilon > 0$. A well-known example is the function \textsc{LeadingOnes}$(x) := \sum_{i=1}^n \prod_{j=1}^i x_j$ that counts the length of the longest prefix of bits set to~1: for every non-optimal search point, an improvement requires the first 0-bit to be flipped.

We show that on all everywhere hard functions $\lambda$ quickly reaches a sufficiently large value such that fitness decreases become unlikely and the \saocl typically behaves like an elitist algorithm. 
We then present simple and easy to use general upper bounds for the runtime of the \saocl on everywhere hard functions that apply for all constant success rates~$s$ and update strengths $F>1$.
More specifically, we show a general upper bound on the expected number of evaluations using the fitness-level method that asymptotically matches the bound obtained for the \oea. For the expected number of generations, we show an upper bound of $O(d + \log(1/\pmin))$ where $d$ is the number of non-optimal fitness values and $\pmin$ is the smallest probability of finding an improvement from any non-optimal search point.

We obtain novel bounds of $O(d)$ expected generations and $O(dn)$ expected evaluations for all everywhere hard unimodal functions with $d+1\ge \log n$ fitness levels. This is $O(n)$ expected generations and $O(n^2)$ expected evaluations for \leadingones. We then introduce a problem class called \onemaxblocks that allows us to tune the difficulty of the fitness levels with a parameter $k$. Varying this parameter changes the behaviour of the function from a \leadingones-like behaviour to a \onemax-like behaviour, allowing us to explore how the difficulty of the function affects the runtime of the \saocl.
Finally, we remark that, when choosing a very small mutation probability of $1/n^{1+\varepsilon}$, $\varepsilon > 0$ constant, every fitness function becomes everywhere hard for the \saocl and then our upper bounds apply to all fitness functions. 


\section{Preliminaries}

We study the expected number of generations and fitness evaluations of \ocl algorithms with self-adjusted offspring population size $\lambda$, or \saocl for short, shown in Algorithm~\ref{alg:saocl}. 
Several mutation operators can be plugged into Line~\ref{line:mutation} to instantiate a particular \saocl. We will consider two operators. 
Standard bit mutations flip each bit independently with a mutation probability $p$. Another mutation operator gaining popularity is the heavy-tailed mutation operator proposed by~\citet{Doerr2017-fastGA}. It performs a standard bit mutation with a mutation probability of $p = \chi/n$ and $\chi$ is chosen randomly in each iteration according to a discrete power-law distribution on $[1..n/2]$ with exponent $\beta > 1$. 

The algorithm maintains a current search point $x$ and an offspring population size~$\lambda$. We define $X_0, X_1,\dots$ as the sequence of states of the algorithm, where $X_t = (x_t, \lambda_t)$ describes the current search point $x_t$ and the offspring population size $\lambda_t$ at generation~$t$. We often omit the subscripts $t$ when the context is obvious.

Note that we regard $\lambda$ to be a real value, so that changes by factors of $1/F$ or $F^{1/s}$ happen on a continuous scale. Following~\cite{Doerr2018,Hevia2021,Hevia2021FOGA}, we assume that, whenever an integer value of $\lambda$ is required, $\lambda$ is rounded to a nearest integer. For the sake of readability, we often write $\lambda$ as a real value even when an integer is required. 

\begin{algorithm}[tbh]
\caption{Self-adjusting $(1,\{F^{1/s}\lambda, \lambda/F \})$~EA.}
\label{alg:saocl}
\SetKwInput{Init}{Initialization}
\SetKwInput{Mut}{Mutation}
\SetKwInput{Sel}{Selection}
\SetKwInput{Opt}{Optimization}
\SetKwInput{Upd}{Update}
\Init{Choose $x \in \{0,1\}^n$ uniformly at random (u.a.r.) and set $\lambda := 1$;}
\Opt{
	\For{$t \in \{1,2,\dots\}$}
		{
		\Mut{
			\For{$i \in \{1,\dots,\lambda \}$}
				{
				$y_i' \in \{0, 1\}^n$ $\leftarrow$ mutate$(x)$\;\label{line:mutation}
				}}
		\Sel{
			 Choose $y \in \{y_1', \dots , y_{\lambda}'\}$ with $f(y) = \max\{f(y_1'), \dots, f(y_\lambda')\}$ u.a.r.\;}
		\Upd{}
			\lIf{$f(y) > f(x)$}{ $x\leftarrow y$; $\lambda \leftarrow \max\{1, \lambda/F\}$}
			\lElse{
				$x\leftarrow y$; $\lambda \leftarrow F^{1/s}\lambda$}
		}
}
\end{algorithm}

Since in all \ocl algorithms selection is performed through comparisons of search points and hence ranks of search points, the absolute fitness values are not relevant.
W.\,l.\,o.\,g.\ we may therefore assume that the domain of any fitness function is taken as integers $\{0, 1, \dots, d\}$ where $d+1 > 1$ is the number of different fitness values and all search points with fitness $d$ are global optima. 
We shall refer to all search points with fitness~$i$, for $0 \le i \le d$ as fitness level~$i$. 

We will use the following notation for all \ocl algorithms. 
\begin{definition}\label{def:probs-and-drifts}
In the context of the \saocl with $\lambda_t = \lambda$
for all $0 \le i < d$ and all search points $x$ with $f(x) < d$ we define:
\begin{align*}
    &\ploss = \prob{f(x_{t+1})< f(x_t)\mid x_t = x} &\\
    &\pimp = \prob{f(x_{t+1})> f(x_t)\mid x_t=x} &\\
    &\Deltaloss = \E{f(x_t)-f(x_{t+1})\mid x_t=x \text{ and } f(x_{t+1})< f(x_t) } &\\
    &\Deltagain = \E{f(x_{t+1})-f(x_t)\mid x_t=x \text{ and } f(x_{t+1})> f(x_t) } &\\
    & s_i = \min_{x}\{p_{x, 1}^+ \mid x_t = x \text{ and } f(x) = i\}
\end{align*}
Here $s_i$ is a lower bound on the probability of one offspring finding an improvement from any search point in fitness level~$i$.
We often refer to the probability $p_{x, 1}^+$ of one offspring improving the current fitness and abbreviate
$p_x^+ := p_{x, 1}^+$ and $p_x^- := p_{x, 1}^-$.
\end{definition}

As in~\cite{Rowe2014}, we call $\Deltagain$ \emph{forward drift} and $\Deltaloss$ \emph{backward drift} and note that they are both at least~1 by definition. 
Now, $p_{x}^+$ is the probability of one offspring finding a better fitness value and ${\pimp = {1 - (1-p_{x}^+)^\lambda}}$ since it is sufficient that one offspring improves the fitness. The probability of a fitness loss is $\ploss = (p_{x}^-)^\lambda$ since all offspring must have worse fitness than their parent. 

We write $\pmin := \min_x\{p_{x}^+ \mid f(x) < d\}$ and ${\pmax := \max_x\{p_{x}^+ \mid f(x) < d\}}$ to denote the minimum and maximum value, resp., for $p_{x}^+$ among all non-optimal search points~$x$. 

We now define the class of functions that we will study throughout this work. The most important characteristic is that there are no easy fitness levels throughout the optimisation.


\begin{definition}\label{def:everywhere-hard}
We say that a function $f$ is \emph{everywhere hard} with respect to a black-box algorithm $\mathcal{A}$ if and only if $\pmax=O\left(n^{-\varepsilon}\right)$ for some constant $0<\varepsilon<1$.
\end{definition}

Owing to non-elitism, the \ocl may decrease its current fitness if all offspring are worse. 
This is often a desired characteristic that allows the algorithm to escape local optima, but if this happens too frequently then the algorithm may not be able to converge to good solutions. Hence, it is important that the probability of an offspring having a lower fitness than its parent is sufficiently small. The probability of this event depends on the mutation operator used. 
The following lemma shows general bounds on transition probabilities for standard bit mutation and heavy-tailed mutations.

\begin{lemma}\label{lem:probability-fallback-SBM-heavytailed}
For all \ocl algorithms using standard bit mutation with a mutation probability in $O(1/n)$ and $n^{-O(1)}$ or heavy-tailed mutation operators with a constant $\beta>1$, there is a constant $\gamma>1$ such that $p_{x}^-\le\gamma^{-1}$ for all non-optimal search points~$x$.

In addition, for all $0 \le i \le d-1$, $p_{x}^+ \ge n^{-O(n)}$.
\end{lemma}
\begin{proof}
Let $C$ denote the event that an offspring is the exact copy of a parent. When using standard bit mutation, if $\chi$ denotes the implicit constant in the bound $O(1/n)$ on the mutation probability,
\begin{align*}
     \prob{C} \ge \left(1-\frac{\chi}{n}\right)^{n} 
     =
     \left(1 - \frac{\chi}{n}\right)^{n-\chi}\left(1 - \frac{\chi}{n}\right)^{\chi}
     \ge e^{-\chi} \cdot \left(1 - \frac{\chi^2}{n}\right) 
\end{align*}
where the inequality follows from~\cite[Corollary~4.6]{DoerrProbabilityChapter2020} and Bernoulli's inequality.
Since $C$ implies that in this generation the current search point cannot worsen, we have $p_{x}^-\le1-\prob{C}$, therefore there is a constant $\gamma>1$ for which $p_{x}^-\le\gamma^{-1}$. 

The probability of the heavy-tailed mutation operator choosing a mutation rate of $1/n$ is
\begin{align*}
    \left(\sum_{i=1}^{n/2} i^{-\beta}\right)^{-1} \ge \left(\sum_{i=1}^{\infty} i^{-\beta}\right)^{-1}
    =\Theta(1).
\end{align*}
Therefore the probability of creating an exact copy of the parent using a heavy-tailed mutation operator is at least $\Theta(1)\cdot\left(1-\frac{1}{n}\right)^{n}=\Theta(1)$ and there is a constant $\gamma>1$ for which $p_{x}^-\le\gamma^{-1}$.

For the second statement, note that the probability of generating a global optimum in one standard bit mutation is at least $\left(n^{-O(1)}\right)^n = n^{-O(n)}$. For heavy-tailed mutations, the probability of choosing a mutation rate of $1/n$ is $\Theta(1)$ as shown above, and then the probability of generating an optimum is at least $n^{-n}$. 
\end{proof}


We conclude this section with a helpful definition for specific $\lambda$-values as follows.
\begin{definition}
\label{def:lambdaone-lambdatwo}
Consider a function $f$ with $d$ fitness values that is everywhere hard for a self-adjusting \ocl with success rate~$s$ that meets the conditions from Lemma~\ref{lem:probability-fallback-SBM-heavytailed}. Let $\gamma$ and $\varepsilon$ be the parameters from Lemma~\ref{lem:probability-fallback-SBM-heavytailed} and Definition~\ref{def:everywhere-hard}, respectively.
Then we define 
$\lambdaone:= 4\max{\left(\log_{\gamma}(2d(s+1)),\log_{\gamma}(n\log n)\right)}$ and $\lambdatwo:=n^{\varepsilon/2}$.
\end{definition}

We consider $\lambdaone$ as a threshold for $\lambda$ such that $\lambda$-values larger than $\lambdaone$ are considered ``safe'' because the probability of a fitness loss is small. We aim to show that the algorithm will typically use values larger than $\lambdaone$ throughout the optimisation. The value $\lambdatwo$ is a threshold for $\lambda$ such that any $\lambda$-value with $\lambda \le \lambdatwo$ has a relatively small success probability. We will show that $\lambda$ has a tendency to increase whenever $\lambda \le \lambdatwo$.


\section{Bounding the number of generations}\label{sec:generations}

We first focus on bounding the expected number of generations as this bound will be used to bound the expected number of function evaluations later on.
The main result of this section is as follows.

\begin{theorem}\label{thm:generations-ewhfunction}
Consider a \saocl using either standard bit mutation with mutation probability $p \in O(1/n) \cap n^{-O(1)}$ or a heavy-tailed mutation operator with a constant $\beta>1$, a constant update strength $F>1$ and a constant success rate $s>0$. 
For all everywhere hard functions $f$ with $d+1=n^{o(\log n)}$ function values the following holds. For every initial search point and every initial offspring population size $\lambda_0$ the \saocl optimises $f$ in an expected number of generations bounded by
\[
O\left(d+\log\left(1/\pmin\right)\right).
\]
\end{theorem}

This result is related to Theorem~3 in~\cite{Lassig2011} which shows the same asymptotic upper bound for the elitist (1+$\{2\lambda, \lambda/2\}$)~EA (i.\,e.\ fixing $F=2$ and $s=1$) on functions on which fitness levels can only become harder as fitness increases. 
Our Theorem~\ref{thm:generations-ewhfunction} applies to everywhere hard functions on which easy and hard fitness levels are mixed in arbitrary ways. And, quite surprisingly, the upper bound only depends on the hardest fitness level.

To bound the number of generations we first need to study how the offspring population size behaves throughout the run. We start by showing that in the beginning of the run $\lambda$ grows quickly.

\begin{lemma}\label{lem:increasing-lambda-ewhfunction}
Consider the \saocl as in Theorem~\ref{thm:generations-ewhfunction}. Let $\tau$ be first generation where $\lambda_{\tau}\ge \lambdatwo$ (cf.\ Definition~\ref{def:lambdaone-lambdatwo}). Then  $\E{\tau}=O(\log \lambdatwo)$. During these $\tau$ generations the algorithm only makes $\lambda_0 + O(\lambdatwo \log \lambdatwo)$ function evaluations in expectation.
\end{lemma}

\begin{proof}
If the initial offspring population size $\lambda_0$ is at least $\lambdatwo$ then $\tau=1$ and $\lambda_0$ evaluations are made. Hence we assume $\lambda_0 < \lambdatwo$.

Following~\cite{Hevia2021}, the parameter $\lambda$ is multiplied in each unsuccessful generation by $F^{1/s}$ and divided by $F$ otherwise. The probability of an unsuccessful generation is at most $\left(1-p^+_{x}\right)^\lambda$ and the probability of a successful generation is at least $1-\left(1-p^+_{x}\right)^\lambda$.

Hence the expected drift of $\log_F(\lambda)$ is at least
\begin{align}
    &\E{\log_F(\lambda_{t+1})-\log_F(\lambda_{t})\mid \lambda_t=\lambda, \lambda_t\le \lambdatwo, x_t = x}\notag\\
    &= \log_F\left(\lambda F^{1/s}\right) \left(1-p^+_{x}\right)^\lambda + \log_F\left(\frac{\lambda}{F}\right) \left(1-\left(1-p^+_{x}\right)^\lambda\right) - \log_F\left(\lambda\right) \notag\\
    & = \left(\log_F\left(\lambda\right) +  \frac{1}{s}\right)\left(1-p^+_{x}\right)^\lambda+\notag\\
    &\hspace{3cm} \left(\log_F\left(\lambda\right)-1\right)\left(1-\left(1-p^+_{x}\right)^\lambda\right)  - \log_F\left(\lambda\right) \notag\\
    & = \frac{s+1}{s} \left(1-p^+_{x}\right)^\lambda - 1
    \ge \frac{s+1}{s} \left(1-\lambda p^+_{x}\right) - 1\notag\\
    &= \frac{1 - (s+1)\lambda p^+_{x}}{s}
    \ge \frac{1-(s+1)\lambdatwo p^+_{x}}{s} = \frac{1}{s} - O\left(n^{-\varepsilon/2}\right)
    \ge \frac{1}{2s}\!\!\!\! \label{eq:potential-drift}
\end{align}
where the last inequality holds for sufficiently large $n$, since $s$ is constant.

We apply additive drift as stated in Theorem~7 in~\citep{Koetzing2019} as it allows for an unbounded state space. We use the potential function
\[
r(\lambda_t)=\log_F(\lambdatwo)-\log(\lambda_t),
\]
which implies that when $r(\lambda_t) \le 0$, $\lambda_t$ is at least $\lambdatwo$. 
By Equation~\eqref{eq:potential-drift}, the drift of $r(\lambda_t)$ is
\begin{align*}
& \E{r(\lambda_t)-r(\lambda_{t+1})\mid \lambda_t=\lambda, \lambda_t\le \lambdatwo}\\
=\;& \E{\log(\lambda_{t+1})-\log(\lambda_{t})\mid \lambda_t=\lambda, \lambda_t\le \lambdatwo}\ge \frac{1}{2s}.
\end{align*}

The initial value $r(\lambda_0)$ is at most $\log_F(\lambdatwo)$ since we assumed $\lambda_0 \le \lambdatwo$.
Now $\tau$ denotes the expected number of generations to reach $r(\lambda_t) \le 0$ for the first time, and $r(\lambda_t) \ge -\frac{1}{s}$ for all $t \le \tau$ since $r(\lambda_{t-1}) > 0$ and $r(\lambda_{t-1}) - r(\lambda_{t}) \le -\log_F(\lambda_{t-1}) + \log_F(F^{1/s}\lambda_{t-1}) = \frac{1}{s}$.
Applying Theorem~7 in~\citep{Koetzing2019} with $\alpha := -\frac{1}{s}$, we obtain
\begin{align*}
    \E{\tau} \le \frac{\log_F(\lambdatwo) + \frac{1}{s}}{1/(2s)}= 2s\log_F(\lambdatwo) + 2 = O(\log \lambdatwo).
\end{align*}
Given that all generations use $\lambda\le \lambdatwo$, the expected number of evaluations during the $O(\log \lambdatwo)$ expected generations is $O(\lambdatwo\log \lambdatwo)$.
\end{proof}

Now we show that, once $\lambda$ reaches a value of at least $\lambdatwo$, the algorithm maintains a large~$\lambda$ with high probability.

\begin{lemma}
\label{lem:probability-of-landslide-decrease-of-lambda-from-lambdaover}
Consider the \saocl as in Theorem~\ref{thm:generations-ewhfunction} at some point of time~$t^*$. For every offspring population size $\lambda_{t^*} \ge \lambdatwo$ the probability that within the next $n^{o(\log n)}$ generations the offspring population size drops below $\lambdaone$ is at most $n^{-\Omega(\log n)}$.
\end{lemma}
\begin{proof}
We first note that $\lambda_t \ge F\lambdatwo$ implies $\lambda_{t+1} \ge \lambdatwo$ with probability~1. Thus, the interval $[\lambdaone, \lambdatwo)$ can only be reached if $\lambda_t < F\lambdatwo$. We may assume that $\lambda_{t^*} < F \lambdatwo$ as otherwise we can simply wait for the first point in time $t^{**}$ where $\lambda_{t^{**}} < F \lambdatwo$ and redefine $t^* := t^{**}$. If no such point in time $t^{**}$ exists, or if $t^{**}-t^* \ge n^{o(\log n)}$, there is nothing to show.

Assuming $\lambda_{t^*} < F\lambdatwo$, we show that an improbably large number of successes are needed for the population size to drop below~$\lambdaone$ before returning to a population size of at least $F\lambdatwo$. 
We define a \emph{trial} as the random time period starting at time $t^*$ and ending when either $\lambda_t < \lambdaone$ or $\lambda_t \ge \lambda_{t^*}$ for some $t > t^*$.
The length of a trial is given by 
\[
    \alpha := \inf\{t - t^* \mid \lambda_t < \lambdaone \vee \lambda_t \ge \lambda_{t^*}, t > t^*\}
\]
and at the end of the trial, either $\lambda_{t^*+\alpha} < \lambdaone$ or $\lambda_{t^*+\alpha} \ge \lambda_{t^*}$ holds.

An important characteristic of the self-adjusting mechanism is that if there 
are $1$ or $0$ successful generations every $\left\lceil s+1\right\rceil$ generations, $\lambda$ will either grow or maintain its previous value, because $\lambda\cdot (F^{1/s})^{\lceil s\rceil} \cdot 1/F \ge \lambda$. Hence, if from the start of a trial there are at most $\kappa$ successful generations during $\left\lceil s+1\right\rceil \kappa$ generations for every $\kappa\in\N$ then $\lambda_{t+\left\lceil s+1\right\rceil \kappa}\ge \lambda_t$, implying that the trail has ended with an offspring population size of at least $\lambda_{t^*}$ and $\alpha\le\left\lceil s+1\right\rceil\kappa$.

We now consider $\kappa^*:=\left\lceil\log_F\left(\frac{\lambdatwo}{\lambdaone}\right)\right\rceil-1$ and show that, to end a trial with $\lambda_{t+\alpha}\le \lambdaone$, more than $\kappa^*$ successful generations are needed. For every $\lambda_{t^*} \ge \lambdatwo$, after $\kappa^*$ consecutive successful generations the offspring population size is
\begin{align*}
    \lambda_{t^*+\kappa^*} = \frac{\lambda_{t^*}}{F^{\kappa^*}} =  \frac{\lambda_{t^*}}{F^{\left\lceil\log_F\left(\frac{\lambdatwo}{\lambdaone}\right)\right\rceil-1}} > \frac{\lambdatwo}{F^{\log_F\left(\frac{\lambdatwo}{\lambdaone}\right)}} = \lambdaone.
\end{align*}
If the successful generations are not consecutive, then the number of successful generations needed to reduce the $\lambda$~value can only increase. Therefore, to reach $\lambda\le \lambdaone$ there must be more than $\kappa^*$ successful generations.

Now, we know that if there are less than $\kappa^*$ successful generations within the first $\left\lceil s+1\right\rceil \kappa^*$ generations of the trial then $\lambda_{t^*+\left\lceil s+1\right\rceil \kappa^*}\ge \lambda_{t^*}$ and we end the trial without dropping below $\lambdaone$. In every generation of a trial, at most $F \lambdatwo$ offspring are created, thus by a union bound, the probability of a successful generation is at most $F\lambdatwo p^+_{x}$.

Let $X$ be the number of successful generations within the first $\left\lceil s+1\right\rceil \kappa^*$ generations of a trial, then $0<\E{X}\le \left\lceil s+1\right\rceil F\lambdatwo p^+_{x} \kappa^*$. Using $\delta:={\kappa^*\E{X}^{-1}-1}$ and Chernoff bounds (Theorem 1.10.1 in~\cite{DoerrProbabilityChapter2020}),
\begin{align*}
    &\Prob{X\ge\kappa^*}
    = \Prob{X\ge\E{X}(1+\delta)}\\
    &\le \exp\left(-\left(\kappa^*\E{X}^{-1}\ln\left(\kappa^*\E{X}^{-1}\right)-\kappa^*\E{X}^{-1}+1\right) \E{X}\right)\\
    &=\exp\left(-\left(\kappa^* \ln\left(\kappa^*\E{X}^{-1}\right)-\kappa^* + \E{X}\right)\right)\\
    &=e^{-\E{X}}\left(\frac{e\E{X}}{\kappa^*}\right)^{\kappa^*}
    \le e^{0}\left(\frac{e\E{X}}{\kappa^*}\right)^{\kappa^*}\\
    &\le \left(e \left\lceil s+1\right\rceil F\lambdatwo p^+_{x}\right)^{\left\lceil\log_F\left(\frac{\lambdatwo}{\lambdaone}\right)\right\rceil-1}
    = n^{-\Omega(\log n)}
\end{align*}
where the last equation uses that the base is $\Theta(\lambdatwo p_{x}^+) = O(n^{\varepsilon/2} \cdot n^{-\varepsilon}) = O(n^{-\Omega(1)})$ and simplifying the exponent using 
\begin{align*}
     \log_F(\lambdatwo/\lambdaone) =\;& \log_F(n^{\varepsilon/2}) - \log_F(\lambdaone)\\
    =\;& \varepsilon/2 \cdot \log_F(n) - o(\log n)) = \Omega(\log n).
\end{align*}
Hence, with probability $n^{-\Omega(\log n)}$ a trial ends with an offspring population size of $\lambda_{t^* + \alpha} \ge \lambda_{t^*}$ and without dropping below $\lambdaone$. Each trial uses at least one generation. By a union bound over $n^{o(\log n)}$ possible number of trials, the probability of reaching ${\lambda\le \lambdaone}$ within $n^{o(\log n)}$ generations is still $n^{-\Omega(\log n)}$.
\end{proof}

Following previous work~\cite{Hevia2021arxiv}, we now define a potential function $g(X_t)$ as a sum of the current search point's fitness and another function $h(\lambda_t)$ that takes into account the current offspring population size: ${g(X_t) = f(x_t) + h(\lambda_t)}$.

\begin{definition}\label{def:potential-function-ewhfunction}
We define the potential function $g(X_t)$ as
\begin{align*}
    g(X_t) = f(x_t) - \frac{s}{s+1} \log_F\left(\max\left(\frac{F^{1/s}}{\pmin\lambda_t}, 1\right)\right).
\end{align*}
\end{definition}

The function $h(\lambda_t) = \frac{s}{s+1} \log_F\left(\max\left(\frac{F^{1/s}}{\pmin\lambda_t}, 1\right)\right)$ is a straightforward generalisation of the approach from~\cite{Hevia2021arxiv} in which the specific value $\pmin = 1/(en)$ was used in the context of \onemax.

Similar to the potential functions used in~\cite{Hevia2021FOGA,Hevia2021arxiv} the potential $g(X_t)$ is always close to the current fitness.

\begin{lemma}\label{lem:relation-fitness-potential-ewhfunction}
For all generations~$t$, the fitness and the potential are related as follows: $f(x_t) - \frac{s}{s+1} \log_F\left(\frac{F^{1/s}}{\pmin}\right) \le g(X_t) \le f(x_t)$. In particular, $g(X_t)=d$ implies ${f(x_t)=d}$.
\end{lemma}
\begin{proof}
The term $\frac{s}{s+1} \log_F\left(\max\left(\frac{F^{1/s}}{\pmin\lambda_t}, 1\right)\right)$ is a non-increasing function in $\lambda_t$ with its minimum being $0$ for ${\lambda_t \ge F^{1/s}/\pmin}$ and its maximum being $\frac{s}{s+1} \log_F\left(\frac{F^{1/s}}{\pmin}\right)$ when $\lambda_t=1$. Hence, $f(x_t) - \frac{s}{s+1} \log_F\left(\frac{F^{1/s}}{\pmin}\right) \le g(X_t) \le f(x_t)$.    
\end{proof}

Given that we have shown that $\lambda$ grows quickly and stays at a large value, we now show that the expected drift of the potential $g(X_t)$ is a positive constant whenever $\lambda$ is at least $\lambdaone$. 
\begin{lemma}\label{lem:potentialDrift-ewhfunction}
Consider the \saocl as in Theorem~\ref{thm:generations-ewhfunction}. Then for every generation $t$ with $f(x_t) < d$ and $\lambda_t \ge \lambdaone$,
\[
    \E{g(X_{t+1})-g(X_{t})\mid X_{t}} \ge \frac{1}{2(s+1)}
\]
for large enough $n$.
This also holds when only considering improvements that increase the fitness by~$1$. 
\end{lemma}
\begin{proof}
We consider only $\lambda \ge \lambdaone>F$, hence, by Lemma~2.8 in~\cite{Hevia2021arxiv}, for all $\lambda \ge \lambdaone$, $\E{g(X_{t+1})-g(X_{t})\mid X_{t}}$ is at least
\begin{equation}
     \left(\Deltagain+h(\lambda/F)-h(\lambda F^{1/s})\right)\pimp + h(\lambda F^{1/s}) -h(\lambda) -\Deltaloss \ploss. \label{eq:potential-with-h-ewhfunction}
\end{equation}
We first consider the case $\lambda_t\le 1/\pmin$ as then $\lambda_{t+1}\le F^{1/s}/\pmin$ and the first term in the maximum of $h(\lambda_{t+1})$ is at least $1$, yielding
$$h(\lambda_{t+1}) = - \frac{s}{s+1} \left(\log_F\left(\frac{F^{1/s}}{\pmin}\right) - \log_F(\lambda_{t+1})\right)<0.$$ 
Hence, $\E{g(X_{t+1}) - g(X_t) \mid X_t, \lambda_t \le 1/\pmin}$ is at least
\begin{align*}
    &\left(\Deltagain-\frac{s}{s+1}\left(\frac{s+1}{s}\right)\right)\pimp +\frac{s}{s+1}\left(\frac{1}{s}\right)-\Deltaloss\ploss\\
    &=\frac{1}{s+1}+\left(\Deltagain-1\right)\pimp-\Deltaloss\ploss.
\end{align*}
By definition $\Deltagain\ge1$, hence
\begin{align*}
    \E{g(X_{t+1}) - g(X_t) \mid X_t, \lambda_t \le 1/\pmin}\ge \frac{1}{s+1}-\Deltaloss\ploss.
\end{align*}
By Lemma~\ref{lem:probability-fallback-SBM-heavytailed}, $p_{x, \lambda_t}^- = (p_x^-)^{\lambda_t} \le \gamma^{-\lambda_t}$. Along with the trivial bound $\Deltaloss \le d$, the right-hand side of the previous inequality is at least
\begin{align*}
    \frac{1}{s+1}-d \gamma^{-\lambda}.
\end{align*}
Since $\lambda_t\ge\lambdaone \ge \log_{\gamma}(2d(s+1))$ the second term is at most $\frac{1}{2(s+1)}$, thus $\E{g(X_{t+1}) - g(X_t) \mid X_t, \lambda_t \le 1/\pmin}\ge\frac{1}{2(s+1)}$.

Finally, for the case $\lambda_t>1/\pmin$, in an unsuccessful generation the penalty term is capped, hence we only know that $h(\lambda F^{1/s}) \ge h(\lambda)$ (which holds with equality if $\lambda_t \ge F^{1/s}/\pmin$). By Equation~\eqref{eq:potential-with-h-ewhfunction}, $\E{g(X_{t+1}) - g(X_t) \mid X_t, \lambda_t > 1/\pmin}$ is at least
\begin{align*}
     &\left(\Deltagain+h(\lambda/F)-h(\lambda)\right)\pimp -\Deltaloss \ploss\\
     &=\left(\Deltagain-\frac{s}{s+1}\right)\pimp -\Deltaloss \ploss.
\end{align*}
By definition of $\Deltagain$, $\Deltagain\ge1$, hence
\begin{align*}
    &\E{g(X_{t+1}) - g(X_t) \mid X_t, \lambda_t \le 1/\pmin}\\
    &\ge \left(\frac{1}{s+1}\right)\pimp -\Deltaloss \ploss.
\end{align*}
$\lambda_t > 1/\pmin$ implies 
$\pimp \ge 1-\left(1-p^+_{x}\right)^{1/\pmin} \ge 1 - \frac{1}{e}$
and by Definition~\ref{def:everywhere-hard}, $\ploss \Deltaloss \le d \gamma^{-1/\pmin}$. Together,
\begin{align*}
    &\E{g(X_{t+1}) - g(X_t) \mid X_t, \lambda_t \le 1/\pmin}
    \\&\ge \left(\frac{1}{s+1}\right)\left(1-\frac{1}{e}\right) -d \gamma^{-1/\pmin}\\
    &=\left(\frac{1}{s+1}\right)\left(1-\frac{1}{e}\right) - o\left(1\right) \ge \frac{1}{2(s+1)}
\end{align*}
where the penultimate step follows from $d \le n^{o(\log n)}$ and 
$\pmin \le \pmax \le n^{-\varepsilon/2}$, which implies $\gamma^{-1/\pmin} \le \gamma^{-n^{\varepsilon/2}} = n^{-\Omega(n^{\varepsilon/2}/\log n)}$. The 
last inequality holds if $n$ is large enough.
\end{proof}

With the previous lemmas we are now able to prove Theorem~\ref{thm:generations-ewhfunction}.
\begin{proof}[Proof of Theorem~\ref{thm:generations-ewhfunction}]
If $\lambda_0<\lambdatwo$ then by Lemma~\ref{lem:increasing-lambda-ewhfunction} in expected $O(\log \lambdatwo)$ generations $\lambda$ will grow to $\lambda \ge \lambdatwo$. Afterwards, by Lemma~\ref{lem:probability-of-landslide-decrease-of-lambda-from-lambdaover}, with probability $1-n^{-\Omega(\log n)}$, the offspring population size will be $\lambda\ge \lambdaone$ in the next $n^{o(\log n)}$ generations. Assuming in the following that this happens, we note that then the drift bound from Lemma~\ref{lem:potentialDrift-ewhfunction} is in force.

Now, similar to \cite{Hevia2021FOGA,Hevia2021arxiv} we bound the number of generations to reach the global optimum using the potential $g(X_t)$. Lemma~\ref{lem:potentialDrift-ewhfunction} shows that the potential has a positive constant drift whenever the optimum has not been found, and by Lemma~\ref{lem:relation-fitness-potential-ewhfunction} if $g(X_t)=d$ then the optimum has been found. Therefore, we can bound the number of generations to find a global optimum by the time it takes for $g(X_t)$ to reach $d$.

To fit the perspective of the additive drift theorem~\cite{He2004} we switch to the function ${\overline{g}(X_t) := d-g(X_t)}$ where $\overline{g}(X_t)=0$ implies $g(X_t) = f(x_t) = d$. The initial value $\overline{g}(X_0)$ is at most 
$d + \frac{s}{s+1} \log_F\left(\frac{F^{1/s}}{\pmin}\right)$ by Lemma~\ref{lem:relation-fitness-potential-ewhfunction}. Using Lemma~\ref{lem:potentialDrift-ewhfunction} and the additive drift theorem, the expected number of generations, assuming no failures, is at most
\begin{align*}
    \E{T} &\le \frac{d + \frac{s}{s+1} \log_F\left(\frac{F^{1/s}}{\pmin}\right)}{ \frac{1}{2(s+1)}}
    = 2(s+1)\cdot d + O\left(\log\left(1/\pmin\right)\right).
\end{align*}
Finally, by Lemma~\ref{lem:probability-fallback-SBM-heavytailed} we have $\pmin\ge n^{-O(n)}$ and thus $\E{T}=O(d+n\log n)$ in case of no failures. Since failures have a probability of $n^{-\Omega(\log n)}$ over $n^{o(\log n)}$ generations and $O(d+n\log n)=n^{o(\log n)}$ using the assumption $d+1 = n^{o(\log n)}$, if we restart the proof every time a failure happens, the expected number of repetitions is $1+n^{-\Omega(\log n)}$ and all additional costs can be absorbed in the previous bounds.
\end{proof}

\section{Bounding the number of evaluations}\label{sec:evaluations}

Now we consider the expected number of fitness evaluations and give the following general result.
\begin{theorem}\label{thm:optimisationTime-ewhfunction}
Consider the \saocl using any mutation operator that ensures $p_x^+ > 0$ for all non-optimal search points~$x$.
Let the update strength $F>1$ and the success rate $s>0$ be constants. Consider an arbitrary everywhere hard function $f$ with $d+1=n^{o(\log n)}$ function values. Then for every initial search point and every initial offspring population size $\lambda_0=O\left(\sum_{i=0}^{d-1}\frac{1}{s_{i}}\right)$ the expected number of evaluations to optimise $f$ is at most
\[
O\left(\sum_{i=0}^{d-1}\frac{1}{s_{i}}\right).
\]
\end{theorem}
The condition $p_x^+ > 0$ is met by standard bit mutation and heavy-tailed mutations.
The term $\sum_{i=0}^{d-1}\frac{1}{s_i}$ equals the fitness-level upper bound for the (1+1)~EA using the same mutation operator as the considered \saocl. A similar result to Theorem~\ref{thm:optimisationTime-ewhfunction} was shown for the (elitist) self-adjusting $(1+\{2\lambda, \lambda/2\})$~EA from~\cite{Lassig2011}. Our result shows that the same bound also applies in the context of non-elitism, if the fitness function is everywhere hard.

The main proof idea is that given that $\lambda$ maintains a large value with high probability throughout the optimisation, the algorithm with high probability behaves as an elitist algorithm. This is shown in the next lemma, adapted from Lemma~3.7 in~\cite{Hevia2021arxiv}.

\begin{lemma}
\label{lem:global-failures}
Consider the \saocl as in Theorem~\ref{thm:optimisationTime-ewhfunction}.
Let $T$ be the first generation in which the optimum is found. Then for all~$t \le T$ in which $\lambda_t \ge \lambdaone$, we have $f(x_{t+1}) \ge f(x_t)$ with probability $1-O(1/(n \log n))$.
\end{lemma}
\begin{proof}
Let $E^t$ denote the event that $\lambda_t < \lambdaone$ or $f(x_{t+1}) \ge f(x_t)$.
Hence we only need to consider $\lambda_t \ge \lambdaone$. 
We note that 
\begin{align*}
    \lambdaone
    &=4\max{\left(\log_{\gamma}(2d(s+1)),\log_{\gamma}(n\log n)\right)}\\
    &\ge2\left(\log_{\gamma}(2d(s+1))+\log_{\gamma}(n\log n)\right).
\end{align*}
Then by Lemma~\ref{lem:probability-fallback-SBM-heavytailed} we have
\begin{align*}
    \Prob{\overline{E^t}} &\le \gamma^{-\lambda_t}
    \le \gamma^{-2\left(\log_{\gamma}(2d(s+1))+\log_{\gamma}(n\log n)\right)}\\
    &= \frac{1}{(2d(s+1)n\log n)^{2}}.
\end{align*}
By a union bound, the probability that this happens in the first $T$ generations is at most 
\begin{align*}
    &\frac{\sum_{t=1}^\infty \Prob{T = t} \cdot t}{(2d(s+1)n\log n)^{2}} 
    = \frac{\E{T}}{(2d(s+1)n\log n)^{2}}.
\end{align*}
By Theorem~\ref{thm:generations-ewhfunction}, this is
\begin{align*}
    &O\left(\frac{d+n\log n}{(dn\log n)^{2}}\right)
    = O\left(\frac{1}{d(n \log n)^2}+\frac{1}{d^2 n\log n}\right)
    =O\left(\frac{1}{n \log n}\right). \qedhere
\end{align*}
\end{proof}

If the algorithm behaves as an elitist algorithm with high probability, we can bound its expected optimisation time by the expected optimisation time of its elitist version, i.\,e.\ a \saopl. 
This argument has already been used in~\cite{Hevia2021arxiv} for the function \onemax, and the expected optimisation time of the elitist \saopl was bounded from above in~\cite[Section 3.2.3]{Hevia2021arxiv}. Inspecting the proofs, we find that the arguments (specifically, Lemma~3.12 and Lemma 3.13 in~\cite{Hevia2021arxiv}) apply to arbitrary fitness functions. Consequently, the proof of Theorem~3.10 in~\cite{Hevia2021arxiv} yields the following more general upper bound.

\begin{theorem}[Generalising Theorem~3.10 in~\cite{Hevia2021arxiv}]
\label{thm:runtime-of-elitist-algorithm}
Consider the elitist \saopl on any function with $d+1$ fitness values, starting with a fitness of $f(x_0) \ge a$. For every integer~$b \le d$, the expected number of evaluations $T(a, b)$  to reach a fitness of at least~$b$ is at most
\[
    T(a, b) \le \lambda_0 \cdot \frac{F}{1-F} + \left(\frac{1}{e} + \frac{1 - F^{-1/s}}{\ln(F^{1/s})}\right) \cdot \frac{F^{\frac{s+1}{s}}-1}{F-1} \sum_{i=a}^{b - 1} \frac{1}{s_{i}}.
\]
\end{theorem}
Note that $T(a, b) \le O\left(\lambda_0 + \sum_{i=a}^{b-1} \frac{1}{s_i}\right)$.
The following lemma bounds the expectation of $\lambda$ at each step~$t$ in order to deal with the case where the \saocl does not behave as an elitist algorithm.
It follows from the proof of Lemma~3.17 in~\cite{Hevia2021arxiv} for \onemax when replacing the specific lower bound of $\frac{n-i}{en}$ on the success probability on \onemax by the general lower bound $\pmin$. 

\begin{lemma}
\label{lem:bound-on-expected-lambda-t}
Consider the \saocl as in Theorem~\ref{thm:optimisationTime-ewhfunction}.
The expected value of $\lambda$ at time~$t$ is
\[
    \E{\lambda_t \mid \lambda_0} \le \lfloor \lambda_0/F^t \rfloor + \frac{1}{\pmin} \cdot \left(F^{1/s} + \frac{F^{1/s}}{\ln F}\right).
\]
\end{lemma}

Now we are in a position to prove Theorem~\ref{thm:optimisationTime-ewhfunction}.

\begin{proof}[Proof of Theorem~\ref{thm:optimisationTime-ewhfunction}]
By Lemma~\ref{lem:increasing-lambda-ewhfunction}, $\lambda$ will grow to $\lambdatwo$ using $O(\lambdatwo\log \lambdatwo)=o(\sqrt{n}\log(n))$ expected evaluations. Afterwards, by Lemma~\ref{lem:probability-of-landslide-decrease-of-lambda-from-lambdaover} the offspring population size will maintain a value of at least $\lambdaone$ with probability $1-n^{-\Omega(\log n)}$ throughout the optimisation and by Lemma~\ref{lem:global-failures} with probability $1-O(1/ (n\log n))$ the algorithm will behave as an elitist algorithm until the optimum is found. Considering the above rare, undesired events as \emph{failures}, 
we define $\E{T^*}$ be the expected time of a run with $\lambda_0=O\left(\sum_{i=0}^{d-1}\frac{1}{s_{i}}\right)$ until either a global optimum is found or a failure occurs. As long as no failure occurs, Theorem~\ref{thm:runtime-of-elitist-algorithm} can be applied with $a := 0$ and thus we obtain $\E{T^*}=O\left(\sum_{i=0}^{d-1}\frac{1}{s_{i}}\right)$.

Since failures have a probability of $O(1/(n \log n))$, we can restart our arguments whenever a failure happens and in expectation there would be at most $1+O(1/(n \log n))$ attempts. 
Arguing as in~\cite[proof of Theorem~3.5]{Hevia2021}, in each restart of the analysis $\lambda_0$ would take the $\lambda$-value at the time of a failure, denoted as $\lambdafail$. By Lemma~\ref{lem:bound-on-expected-lambda-t},
\[
    \E{\lambdafail} \le \lambda_0 + O\left(\frac{1}{\pmin}\right)
    \le O\left(\sum_{i=0}^{d-1}\frac{1}{s_{i}}\right)
\]
and this term, multiplied by $O(1/(n \log n))$, can easily be absorbed in our claimed upper bound. 
%
%
\end{proof}

\section{Bounds on unimodal functions}

We now show how to apply Theorems~\ref{thm:generations-ewhfunction} and~\ref{thm:optimisationTime-ewhfunction} to obtain novel bounds on the expected optimisation time of the \saocl on \emph{everywhere hard} unimodal functions, the benchmark function \leadingones and a new function class \onemaxblocks that allows us to vary the difficulty of the easiest fitness levels.


\begin{definition}
\label{LeadingOnes-variants}
Let $k, n \in \mathbb{N}$ such that $n/k \in \mathbb{N}$, then for $x = (x_1 \dots x_n) \in \{0, 1\}^n$ we define:
\begin{itemize}[leftmargin=1.6em]
    \item $\leadingones(x) := \sum_{i=1}^{n}  \prod_{j=1}^{i} x_j$,
    \item ${\onemaxblocks(x) := \sum_{j=1}^{\lfloor n/k \rfloor} \left(\prod_{i=1}^{(j-1)k} x_i\right) \cdot \sum_{i=(j-1)k+1}^{jk} x_i}$,
    \item a fitness function is called unimodal if every non-optimal search point~$x$ has a Hamming neighbour (a search point that only differs in one bit from~$x$) with strictly larger fitness.
\end{itemize}
\end{definition}

Recall that \textsc{LeadingOnes} returns the number of \onebits in the longest prefix that only contains ones.
The proposed function class, \onemaxblocks, has a similar structure. It is comprised of \emph{blocks} of $k$ bits. A block is \emph{complete} if it only contains \onebits and \emph{incomplete} otherwise. The function returns the number of \onebits in the longest prefix of completed blocks plus the number of \onebits in the first incomplete block.
Evolutionary algorithms typically optimise this function by optimising each \onemax-like block of size $k$ from left to right until the global optimum $1^n$ is reached. We can tune the maximum success probability $\pmax$ by assigning different values to~$k$. If $k=1$ the function equals \leadingones where $\pmax=\Theta(1/n)$. Increasing $k$ increases the maximum success probability. If $k=n$ the function equals \onemax and ${\pmax=\Theta(1)}$. 



\begin{theorem}\label{thm:optimisationTime-unimodal}
Let $s>0$ and $F>1$ be constants. The expected number of generations and evaluations of the \saocl using standard bit mutation with mutation probability $\chi/n$, $\chi=\Theta(1)$, or heavy-tailed mutations with constant $\beta>1$ is at most
\begin{enumerate}
    \item $O(d)$ expected generations and $O(dn)$ expected evaluations on all unimodal functions with $d+1 \ge \log n$ fitness values that are everywhere hard for the considered algorithm\label{it:unimodal},
    \item $O(n)$ expected generations and $O(n^2)$ expected evaluations on \leadingones\label{it:leadingones}, and
    \item $O(n)$ expected generations and $O\left(\frac{n^2\log k}{k}\right)$ expected evaluations on \onemaxblocks with $1 < k\le n^{1-\varepsilon}$ for some constant $0<\varepsilon<1$ and, additionally, $k \le n^{\beta-1-\varepsilon}$ if heavy-tailed mutations are used. \label{it:leadingblocks}
    
\end{enumerate}
\end{theorem}
\begin{proof}
The set of everywhere hard unimodal functions by definition meet the conditions in Definition~\ref{def:everywhere-hard}, therefore we can apply Theorems~\ref{thm:generations-ewhfunction} and~\ref{thm:optimisationTime-ewhfunction} directly. We only need to bound $\pmin$. 
Given that every search point has a strictly better Hamming neighbour, the success probability of all fitness levels is at least as large as the probability of flipping only one specific bit. For standard bit mutations with mutation probability $\chi/n$, this is at least
\begin{align*}
    \pmin\ge\frac{\chi}{n}\left(1-\frac{\chi}{n}\right)^{n-1}\ge\frac{\chi}{e^{\chi}n} \cdot \left(1 - \frac{\chi}{n}\right)^{1-\chi} = \Omega(1/n).
\end{align*}
For heavy-tailed mutations we also have $\pmin = \Omega(1/n)$ as there is a constant probability of choosing a mutation rate of $1/n$.
Applying Theorems~\ref{thm:generations-ewhfunction} and~\ref{thm:optimisationTime-ewhfunction} yields $O(d + \log n) = O(d)$ generations and $O(dn)$ evaluations in expectation.

For the statement on \leadingones we need to show that the function is everywhere hard for the considered algorithms. A necessary condition for an offspring to be better than the parent is to flip the leftmost \zerobit, hence the success probability of standard bit mutation on all fitness levels is bounded by 
$p^+_{x}\le\frac{\chi}{n}$ and \leadingones meets the conditions of an everywhere hard unimodal function. Thus, it is covered by the previous statement with $d := n$. 
For heavy-tailed mutations, a much larger mutation rate might be used, hence we need to be more careful. 
Heavy-tailed mutation chooses a mutation probability $\chi^*/n$ according to a power-law distribution with parameter $\beta$, truncated at $n/2$. The probability of choosing a certain mutation rate $\chi/n$ is 
\begin{align*}
\prob{\chi^*=\chi}=\frac{\chi^{-\beta}}{\sum_{j=1}^{n/2}j^{-\beta}} \le \frac{\chi^{-\beta}}{\zeta(\beta)-\rho},
\end{align*}
where the inequality is taken from~\cite{Doerr2017-fastGA} and $\zeta(\beta)$ is the Riemann zeta function $\zeta$ evaluated at $\beta$ and $\rho=\frac{\beta}{\beta-1}\left(\frac{n}{2}\right)^{-\beta+1}=o(1)$. 
Given a mutation probability of $\chi/n$ (where $\chi$ is no longer restricted to a constant), the probability of flipping the first \zerobit is $\chi/n$. Hence, 
\begin{align*}
    p^+&\le\sum_{\chi=1}^{n/2}\frac{\chi^{-\beta}}{\zeta(\beta)-\rho} \cdot \frac{\chi}{n}
    = \frac{\sum_{\chi=1}^{n/2}\chi^{1-\beta}}{n(\zeta(\beta)-\rho)} 
    = \Theta\left(\frac{1}{n}\right) \cdot \sum_{\chi=1}^{n/2}\chi^{1-\beta}.
\end{align*}
Since $\beta > 1$, $\chi^{1-\beta}$ is strictly decreasing with~$\chi$ and, using $\chi^{1-\beta} \le \int_{\chi-1}^{\chi} x^{1-\beta} \;\mathrm{d}x$, we can bound the sum by an integral:
\[
    \sum_{\chi=1}^{n/2} \chi^{1-\beta} \le \int_{0}^{n/2} \chi^{1-\beta} \;\mathrm{d}\chi = \frac{(n/2)^{2-\beta}}{2-\beta}.
\]
Together, we get
\begin{align*}
    p^+&\le \Theta\left(\frac{1}{n}\right) \cdot \frac{(n/2)^{2-\beta}}{2-\beta} =
    O(n^{1-\beta})
\end{align*}
fitting the definition of everywhere hardness for ${\varepsilon := \beta-1 > 0}$.

For \onemaxblocks, a search point of fitness $i < n$ has ${i \bmod k}$ \onebits in its first incomplete block. All non-optimal search points can only be improved by flipping at least one \zerobit in the first incomplete block. Hence, a necessary condition for an offspring to be better than the parent is to flip one of the $k-(i\;\mathrm{mod}\;{k}) \le k$ \zerobits in the first incomplete block. 
Hence, all success probabilities can be bounded by $k$ times the bound on the success probability for \leadingones.
Thus, for all non-optimal~$x$, $p^+_x \le \chi \cdot k/n \le \chi n^{-\varepsilon}$ for standard bit mutations using the assumption $k \le n^{1-\varepsilon}$
and $p_x^+ \le O(k \cdot n^{1-\beta}) = O(n^{-\varepsilon})$ for heavy-tailed mutation using the assumption $k \le n^{\beta-1-\varepsilon}$.
For both operators, \onemaxblocks meets the conditions of an everywhere hard unimodal function. 

A sufficient condition for an offspring to be better than its parent of fitness~$i$ is to flip only one of the $k-(i\;\mathrm{mod}\;{k})$ \zerobits in the first incomplete block. Standard bit mutations do this with probability
\begin{align*}
    s_i\ge\frac{\chi(k-(i\;\mathrm{mod}\;{k}))}{n}\left(1-\frac{\chi}{n}\right)^{n-1}.
\end{align*}
By Theorem~\ref{thm:generations-ewhfunction}, the \saocl optimises \onemaxblocks with $k\le n^{1-\varepsilon}$ in $O(n)$ generations.
By Theorem~\ref{thm:optimisationTime-ewhfunction} the expected number of evaluations is at most
\begin{align*}
    O\left(\sum_{i=1}^{n}\frac{\left(1 - \frac{\chi}{n}\right)^{-n+1}n}{\chi(k-(i\;\mathrm{mod}\;{k}))}\right)
    = O\left(\frac{n}{k}\cdot \sum_{j=1}^{k}\frac{n}{\chi j}\right)
    = O\left(\frac{n^2\log k}{k}\right).
\end{align*}
If $k=1$ \onemaxblocks equals \leadingones and the expected number of evaluations is $O(n^2)$. 
For heavy-tailed mutations a mutation rate of $1/n$ is used with constant probability, hence the above asymptotic probability bounds apply as well.
\end{proof}

\section{Very Small Mutation Rates Make All Functions Everywhere Hard}

In this short section 
we remark that, if the \saocl uses standard bit mutation with mutation rate $1/n^{1+\varepsilon}$ for some constant $0<\varepsilon$ then all functions are everywhere hard since any mutation will create a copy of the parent with probability at least $1-n^{-\varepsilon}$ and thus the probability of an offspring improving the fitness is at most $\pmax \le n^{-\varepsilon}$. Thus, functions like \onemax, where large constant values of $s$ result in exponential runtimes (Theorem~4.1 and~4.4 in~\cite{Hevia2021arxiv}), can be solved in polynomial expected time for arbitrary constant~$s$. For example, the \saocl with every constant $s>1$ can solve \onemax in $O(n)$ expected generations and $O(n^{1+\varepsilon}\log n)$ expected evaluations. 


\begin{corollary}
Let $0<\varepsilon<1$, the update strength $F>1$ and the success rate $s>0$ be constants. For every function $f$ with $d+1=n^{o(\log n)}$ fitness levels, the \saocl using standard bit mutation with mutation rate $1/n^{1+\varepsilon}$ optimises $f$ in $O\left(d+\log\left(1/\pmin\right)\right)$ expected generations and $O\left(\sum_{i=1}^{d-1}\frac{1}{s_{i}}\right)$ expected evaluations.
\end{corollary}

The conclusion is that making an algorithm less efficient (in the sense of introducing large self-loops) can improve performance as the algorithm shows a more stable search behaviour. Similar observations were made before for the \ocl with fixed~$\lambda$~\cite{Rowe2014} and, in a wider sense, in evolution with partial information~\cite{Dang2016}.

\section{Conclusions}

We have shown that the non-elitist \saocl is not affected by the choice of the success rate (from positive constants) if the problem in hand is everywhere hard, that is, improvements are always found with a probability of at most $n^{-\varepsilon}$. This is in stark contrast to functions with easy slopes like \onemax, on which the \saocl takes exponential time if $s$ is too large, since frequent improvements drive down the population size. 

Our analysis extends previous work~\cite{Hevia2021,Hevia2021FOGA} on \onemax and \cliff to all everywhere hard functions. Moreover, our results apply to both standard bit mutation as well as heavy-tailed mutations. 
The expected number of evaluations is bounded by the same fitness-level upper bound as known for the \oea using the same mutation operator. Self-adjusting the offspring population size drastically reduces the number of generations to just $O(d + \log(1/\pmin))$, that is, roughly to the number of fitness values, improving and generalising previous results~\cite{Lassig2011}.
As a byproduct of our analysis, we have also shown an upper bound for the expected number of evaluations of the elitist \saopl on arbitrary fitness functions.

Although our results show that the \saocl is robust with respect to the choice of its hyper-parameters on hard functions, it remains an open problem how to self-adjust the offspring population size~$\lambda$ for the \ocl in such a way that the algorithm performs well on both easy and hard functions (independent of the choice of its hyper-parameters) without worsening the runtime guarantees obtained for the multimodal function \cliff.

\begin{acks}
This research has been supported by CONACYT under the grant no. 739621 and registration no. 843375.
\end{acks}

\balance


\begin{thebibliography}{30}


\ifx \showCODEN    \undefined \def \showCODEN     #1{\unskip}     \fi
\ifx \showDOI      \undefined \def \showDOI       #1{#1}\fi
\ifx \showISBNx    \undefined \def \showISBNx     #1{\unskip}     \fi
\ifx \showISBNxiii \undefined \def \showISBNxiii  #1{\unskip}     \fi
\ifx \showISSN     \undefined \def \showISSN      #1{\unskip}     \fi
\ifx \showLCCN     \undefined \def \showLCCN      #1{\unskip}     \fi
\ifx \shownote     \undefined \def \shownote      #1{#1}          \fi
\ifx \showarticletitle \undefined \def \showarticletitle #1{#1}   \fi
\ifx \showURL      \undefined \def \showURL       {\relax}        \fi
\providecommand\bibfield[2]{#2}
\providecommand\bibinfo[2]{#2}
\providecommand\natexlab[1]{#1}
\providecommand\showeprint[2][]{arXiv:#2}

\bibitem[\protect\citeauthoryear{Badkobeh, Lehre, and Sudholt}{Badkobeh
  et~al\mbox{.}}{2014}]%
        {Badkobeh2014}
\bibfield{author}{\bibinfo{person}{Golnaz Badkobeh},
  \bibinfo{person}{Per~Kristian Lehre}, {and} \bibinfo{person}{Dirk Sudholt}.}
  \bibinfo{year}{2014}\natexlab{}.
\newblock \showarticletitle{Unbiased Black-Box Complexity
  of Parallel Search}. In \bibinfo{booktitle}{{\em Proceedings of Parallel
  Problem Solving from Nature -- PPSN XIII}}. \bibinfo{publisher}{Springer},
  \bibinfo{pages}{892--901}.
\newblock


\bibitem[\protect\citeauthoryear{B\"{o}ttcher, Doerr, and Neumann}{B\"{o}ttcher
  et~al\mbox{.}}{2010}]%
        {Boettcher2010}
\bibfield{author}{\bibinfo{person}{S\"{u}ntje B\"{o}ttcher},
  \bibinfo{person}{Benjamin Doerr}, {and} \bibinfo{person}{Frank Neumann}.}
  \bibinfo{year}{2010}\natexlab{}.
\newblock \showarticletitle{Optimal Fixed and Adaptive Mutation Rates for the
  {LeadingOnes} Problem}. In \bibinfo{booktitle}{{\em Proceedings of Parallel
  Problem Solving from Nature -- PPSN XI}}, Vol.~\bibinfo{volume}{6238}.
  \bibinfo{publisher}{Springer}, \bibinfo{pages}{1--10}.
\newblock


\bibitem[\protect\citeauthoryear{Corus, Oliveto, and Yazdani}{Corus
  et~al\mbox{.}}{2020}]%
        {Corus2020}
\bibfield{author}{\bibinfo{person}{Dogan Corus}, \bibinfo{person}{Pietro~S.
  Oliveto}, {and} \bibinfo{person}{Donya Yazdani}.}
  \bibinfo{year}{2020}\natexlab{}.
\newblock \showarticletitle{When Hypermutations and Ageing Enable Artificial
  Immune Systems to Outperform Evolutionary Algorithms}.
\newblock \bibinfo{journal}{{\em Theoretical Computer Science\/}}
  \bibinfo{volume}{832} (\bibinfo{year}{2020}), \bibinfo{pages}{166--185}.
\newblock


\bibitem[\protect\citeauthoryear{Dang and Lehre}{Dang and Lehre}{2016}]%
        {Dang2016}
\bibfield{author}{\bibinfo{person}{Duc-Cuong Dang} {and}
  \bibinfo{person}{Per~Kristian Lehre}.} \bibinfo{year}{2016}\natexlab{}.
\newblock \showarticletitle{Runtime Analysis of Non-elitist Populations: From
  Classical Optimisation to Partial Information}.
\newblock \bibinfo{journal}{{\em Algorithmica\/}}  \bibinfo{volume}{75}
  (\bibinfo{year}{2016}), \bibinfo{pages}{428–461}.
\newblock


\bibitem[\protect\citeauthoryear{Doerr}{Doerr}{2020}]%
        {DoerrProbabilityChapter2020}
\bibfield{author}{\bibinfo{person}{Benjamin Doerr}.}
  \bibinfo{year}{2020}\natexlab{}.
\newblock \bibinfo{booktitle}{{\em Probabilistic Tools for the Analysis of
  Randomized Optimization Heuristics}}.
\newblock \bibinfo{publisher}{Springer}, \bibinfo{pages}{1--87}.
\newblock


\bibitem[\protect\citeauthoryear{Doerr and Doerr}{Doerr and Doerr}{2018}]%
        {Doerr2018}
\bibfield{author}{\bibinfo{person}{Benjamin Doerr} {and}
  \bibinfo{person}{Carola Doerr}.} \bibinfo{year}{2018}\natexlab{}.
\newblock \showarticletitle{Optimal Static and Self-Adjusting Parameter Choices
  for the (1+($\lambda$,$\lambda$)) Genetic Algorithm}.
\newblock \bibinfo{journal}{{\em Algorithmica\/}} \bibinfo{volume}{80},
  \bibinfo{number}{5} (\bibinfo{year}{2018}), \bibinfo{pages}{1658--1709}.
\newblock


\bibitem[\protect\citeauthoryear{Doerr and Doerr}{Doerr and Doerr}{2020}]%
        {DoerrSurvey2020}
\bibfield{author}{\bibinfo{person}{Benjamin Doerr} {and}
  \bibinfo{person}{Carola Doerr}.} \bibinfo{year}{2020}\natexlab{}.
\newblock \bibinfo{booktitle}{{\em Theory of Parameter Control for Discrete
  Black-Box Optimization: Provable Performance Gains Through Dynamic Parameter
  Choices}}.
\newblock \bibinfo{publisher}{Springer}, \bibinfo{pages}{271--321}.
\newblock


\bibitem[\protect\citeauthoryear{Doerr, Doerr, and Ebel}{Doerr
  et~al\mbox{.}}{2015}]%
        {Doerr2015}
\bibfield{author}{\bibinfo{person}{Benjamin Doerr}, \bibinfo{person}{Carola
  Doerr}, {and} \bibinfo{person}{Franziska Ebel}.}
  \bibinfo{year}{2015}\natexlab{}.
\newblock \showarticletitle{From Black-Box Complexity to Designing New Genetic
  Algorithms}. In \bibinfo{booktitle}{{\em Theoretical Computer Science}},
  Vol.~\bibinfo{volume}{567}. \bibinfo{pages}{87--104}.
\newblock


\bibitem[\protect\citeauthoryear{Doerr, Doerr, and K{\"{o}}tzing}{Doerr
  et~al\mbox{.}}{2018}]%
        {DoerrDK2018}
\bibfield{author}{\bibinfo{person}{Benjamin Doerr}, \bibinfo{person}{Carola
  Doerr}, {and} \bibinfo{person}{Timo K{\"{o}}tzing}.}
  \bibinfo{year}{2018}\natexlab{}.
\newblock \showarticletitle{Static and Self-Adjusting Mutation Strengths for
  Multi-valued Decision Variables}.
\newblock \bibinfo{journal}{{\em Algorithmica\/}} \bibinfo{volume}{80},
  \bibinfo{number}{5} (\bibinfo{year}{2018}), \bibinfo{pages}{1732--1768}.
\newblock


\bibitem[\protect\citeauthoryear{Doerr, Doerr, and Lengler}{Doerr
  et~al\mbox{.}}{2021}]%
        {DoerrSAMut2021}
\bibfield{author}{\bibinfo{person}{Benjamin Doerr}, \bibinfo{person}{Carola
  Doerr}, {and} \bibinfo{person}{Johannes Lengler}.}
  \bibinfo{year}{2021}\natexlab{}.
\newblock \showarticletitle{Self-Adjusting Mutation Rates with Provably Optimal
  Success Rules}.
\newblock \bibinfo{journal}{{\em Algorithmica\/}} \bibinfo{volume}{83},
  \bibinfo{number}{10} (\bibinfo{year}{2021}), \bibinfo{pages}{3108--3147}.
\newblock


\bibitem[\protect\citeauthoryear{Doerr, Doerr, and Yang}{Doerr
  et~al\mbox{.}}{2020}]%
        {Doerr20201}
\bibfield{author}{\bibinfo{person}{Benjamin Doerr}, \bibinfo{person}{Carola
  Doerr}, {and} \bibinfo{person}{Jing Yang}.} \bibinfo{year}{2020}\natexlab{}.
\newblock \showarticletitle{Optimal Parameter Choices via Precise Black-Box
  Analysis}.
\newblock \bibinfo{journal}{{\em Theoretical Computer Science\/}}
  \bibinfo{volume}{801} (\bibinfo{year}{2020}), \bibinfo{pages}{1 -- 34}.
\newblock


\bibitem[\protect\citeauthoryear{Doerr, Gie{\ss}en, Witt, and Yang}{Doerr
  et~al\mbox{.}}{2019}]%
        {Doerr2019opl}
\bibfield{author}{\bibinfo{person}{Benjamin Doerr}, \bibinfo{person}{Christian
  Gie{\ss}en}, \bibinfo{person}{Carsten Witt}, {and} \bibinfo{person}{Jing
  Yang}.} \bibinfo{year}{2019}\natexlab{}.
\newblock \showarticletitle{The (1+ $\lambda$) {Evolutionary} {Algorithm} with
  Self-Adjusting Mutation Rate}.
\newblock \bibinfo{journal}{{\em Algorithmica\/}} \bibinfo{volume}{81},
  \bibinfo{number}{2} (\bibinfo{year}{2019}), \bibinfo{pages}{593--631}.
\newblock


\bibitem[\protect\citeauthoryear{Doerr, Le, Makhmara, and Nguyen}{Doerr
  et~al\mbox{.}}{2017}]%
        {Doerr2017-fastGA}
\bibfield{author}{\bibinfo{person}{Benjamin Doerr}, \bibinfo{person}{Huu~Phuoc
  Le}, \bibinfo{person}{R{\'e}gis Makhmara}, {and} \bibinfo{person}{Ta~Duy
  Nguyen}.} \bibinfo{year}{2017}\natexlab{}.
\newblock \showarticletitle{Fast Genetic Algorithms}. In
  \bibinfo{booktitle}{{\em Proceedings of the Genetic and Evolutionary
  Computation Conference}} {\em (\bibinfo{series}{GECCO ’17})}.
  \bibinfo{publisher}{ACM}, \bibinfo{pages}{777--784}.
\newblock


\bibitem[\protect\citeauthoryear{He and Yao}{He and Yao}{2004}]%
        {He2004}
\bibfield{author}{\bibinfo{person}{Jun He} {and} \bibinfo{person}{Xin Yao}.}
  \bibinfo{year}{2004}\natexlab{}.
\newblock \showarticletitle{A Study of Drift Analysis for Estimating
  Computation Time of Evolutionary algorithms}.
\newblock \bibinfo{journal}{{\em Natural Computing\/}} \bibinfo{volume}{3},
  \bibinfo{number}{1} (\bibinfo{year}{2004}), \bibinfo{pages}{21--35}.
\newblock


\bibitem[\protect\citeauthoryear{Hevia~Fajardo and Sudholt}{Hevia~Fajardo and
  Sudholt}{2020}]%
        {Hevia2020}
\bibfield{author}{\bibinfo{person}{Mario~Alejandro Hevia~Fajardo} {and}
  \bibinfo{person}{Dirk Sudholt}.} \bibinfo{year}{2020}\natexlab{}.
\newblock \showarticletitle{On the Choice of the Parameter Control Mechanism in
  the (1+($\lambda$, $\lambda$)) {Genetic} {Algorithm}}. In
  \bibinfo{booktitle}{{\em Proceedings of the Genetic and Evolutionary
  Computation}} {\em (\bibinfo{series}{GECCO '20})}. \bibinfo{publisher}{ACM},
  \bibinfo{pages}{832–840}.
\newblock


\bibitem[\protect\citeauthoryear{Hevia~Fajardo and Sudholt}{Hevia~Fajardo and
  Sudholt}{2021a}]%
        {Hevia2021FOGA}
\bibfield{author}{\bibinfo{person}{Mario~Alejandro Hevia~Fajardo} {and}
  \bibinfo{person}{Dirk Sudholt}.} \bibinfo{year}{2021}\natexlab{a}.
\newblock \showarticletitle{Self-Adjusting Offspring Population Sizes
  Outperform Fixed Parameters on the Cliff Function}. In
  \bibinfo{booktitle}{{\em Proceedings of the 16th Workshop on Foundations of
  Genetic Algorithms}} {\em (\bibinfo{series}{FOGA '21})}.
  \bibinfo{publisher}{ACM}, \bibinfo{address}{New York, NY, USA},
  \bibinfo{pages}{5:1--5:15}.
\newblock


\bibitem[\protect\citeauthoryear{Hevia~Fajardo and Sudholt}{Hevia~Fajardo and
  Sudholt}{2021b}]%
        {Hevia2021arxiv}
\bibfield{author}{\bibinfo{person}{Mario~Alejandro Hevia~Fajardo} {and}
  \bibinfo{person}{Dirk Sudholt}.} \bibinfo{year}{2021}\natexlab{b}.
\newblock \showarticletitle{Self-Adjusting Population Sizes for Non-Elitist
  {Evolutionary} {Algorithms}: Why Success Rates Matter}.
\newblock \bibinfo{journal}{{\em ArXiv e-prints\/}} (\bibinfo{year}{2021}).
\newblock


\bibitem[\protect\citeauthoryear{Hevia~Fajardo and Sudholt}{Hevia~Fajardo and
  Sudholt}{2021c}]%
        {Hevia2021}
\bibfield{author}{\bibinfo{person}{Mario~Alejandro Hevia~Fajardo} {and}
  \bibinfo{person}{Dirk Sudholt}.} \bibinfo{year}{2021}\natexlab{c}.
\newblock \showarticletitle{Self-Adjusting Population Sizes for Non-Elitist
  {Evolutionary} {Algorithms}: Why Success Rates Matter}. In
  \bibinfo{booktitle}{{\em Proceedings of the Genetic and Evolutionary
  Computation Conference}}. \bibinfo{publisher}{ACM},
  \bibinfo{pages}{1151–1159}.
\newblock


\bibitem[\protect\citeauthoryear{J{\"a}gersk{\"u}pper and
  Storch}{J{\"a}gersk{\"u}pper and Storch}{2007}]%
        {Jagerskupper2007a}
\bibfield{author}{\bibinfo{person}{Jens J{\"a}gersk{\"u}pper} {and}
  \bibinfo{person}{Tobias Storch}.} \bibinfo{year}{2007}\natexlab{}.
\newblock \showarticletitle{When the Plus Strategy Outperforms the Comma
  Strategy and When Not}. In \bibinfo{booktitle}{{\em Proceedings of the IEEE
  Symposium on Foundations of Computational Intelligence, FOCI 2007}}.
  \bibinfo{publisher}{IEEE}, \bibinfo{pages}{25--32}.
\newblock


\bibitem[\protect\citeauthoryear{Kern, M{\"u}ller, Hansen, B{\"u}che, Ocenasek,
  and Koumoutsakos}{Kern et~al\mbox{.}}{2004}]%
        {KMH04}
\bibfield{author}{\bibinfo{person}{Stefan Kern}, \bibinfo{person}{Sibylle~D.
  M{\"u}ller}, \bibinfo{person}{Nikolaus Hansen}, \bibinfo{person}{Dirk
  B{\"u}che}, \bibinfo{person}{Jiri Ocenasek}, {and} \bibinfo{person}{Petros
  Koumoutsakos}.} \bibinfo{year}{2004}\natexlab{}.
\newblock \showarticletitle{Learning Probability Distributions in Continuous
  Evolutionary Algorithms -- a Comparative Review}.
\newblock \bibinfo{journal}{{\em Natural Computing\/}} \bibinfo{volume}{3},
  \bibinfo{number}{1} (\bibinfo{year}{2004}), \bibinfo{pages}{77--112}.
\newblock


\bibitem[\protect\citeauthoryear{K{\"o}tzing and Krejca}{K{\"o}tzing and
  Krejca}{2019}]%
        {Koetzing2019}
\bibfield{author}{\bibinfo{person}{Timo K{\"o}tzing} {and}
  \bibinfo{person}{Martin~S. Krejca}.} \bibinfo{year}{2019}\natexlab{}.
\newblock \showarticletitle{First-hitting times under drift}.
\newblock \bibinfo{journal}{{\em Theoretical Computer Science\/}}
  \bibinfo{volume}{796} (\bibinfo{year}{2019}), \bibinfo{pages}{51--69}.
\newblock


\bibitem[\protect\citeauthoryear{L\"{a}ssig and Sudholt}{L\"{a}ssig and
  Sudholt}{2011}]%
        {Lassig2011}
\bibfield{author}{\bibinfo{person}{J\"{o}rg L\"{a}ssig} {and}
  \bibinfo{person}{Dirk Sudholt}.} \bibinfo{year}{2011}\natexlab{}.
\newblock \showarticletitle{Adaptive Population Models for Offspring
  Populations and Parallel Evolutionary Algorithms}. In
  \bibinfo{booktitle}{{\em Proceedings of the 11th Workshop on Foundations of
  Genetic Algorithms}} {\em (\bibinfo{series}{FOGA '11})}.
  \bibinfo{publisher}{ACM}, \bibinfo{pages}{181--192}.
\newblock


\bibitem[\protect\citeauthoryear{Lehre and Sudholt}{Lehre and Sudholt}{2020}]%
        {LehreD2020}
\bibfield{author}{\bibinfo{person}{Per~Kristian Lehre} {and}
  \bibinfo{person}{Dirk Sudholt}.} \bibinfo{year}{2020}\natexlab{}.
\newblock \showarticletitle{Parallel Black-Box Complexity with Tail Bounds}.
\newblock \bibinfo{journal}{{\em IEEE Transactions on Evolutionary
  Computation\/}} \bibinfo{volume}{24}, \bibinfo{number}{6}
  (\bibinfo{year}{2020}), \bibinfo{pages}{1010--1024}.
\newblock


\bibitem[\protect\citeauthoryear{Lissovoi, Oliveto, and Warwicker}{Lissovoi
  et~al\mbox{.}}{2019}]%
        {LissovoiMultimodal2019}
\bibfield{author}{\bibinfo{person}{Andrei Lissovoi}, \bibinfo{person}{Pietro~S.
  Oliveto}, {and} \bibinfo{person}{John~Alasdair Warwicker}.}
  \bibinfo{year}{2019}\natexlab{}.
\newblock \showarticletitle{On the Time Complexity of Algorithm Selection
  Hyper-Heuristics for Multimodal Optimisation}. In \bibinfo{booktitle}{{\em
  Proceedings of the AAAI Conference on Artificial Intelligence}},
  Vol.~\bibinfo{volume}{33}. \bibinfo{pages}{2322--2329}.
\newblock


\bibitem[\protect\citeauthoryear{Mambrini and Sudholt}{Mambrini and
  Sudholt}{2015}]%
        {Mambrini2015}
\bibfield{author}{\bibinfo{person}{Andrea Mambrini} {and} \bibinfo{person}{Dirk
  Sudholt}.} \bibinfo{year}{2015}\natexlab{}.
\newblock \showarticletitle{Design and Analysis of Schemes for Adapting
  Migration Intervals in Parallel Evolutionary Algorithms}.
\newblock \bibinfo{journal}{{\em Evolutionary Computation\/}}
  \bibinfo{volume}{23}, \bibinfo{number}{4} (\bibinfo{year}{2015}),
  \bibinfo{pages}{559--582}.
\newblock


\bibitem[\protect\citeauthoryear{Paix{\~a}o, P{\'e}rez~Heredia, Sudholt, and
  Trubenov{\'a}}{Paix{\~a}o et~al\mbox{.}}{2017}]%
        {Paixao2016}
\bibfield{author}{\bibinfo{person}{Tiago Paix{\~a}o}, \bibinfo{person}{Jorge
  P{\'e}rez~Heredia}, \bibinfo{person}{Dirk Sudholt}, {and}
  \bibinfo{person}{Barbora Trubenov{\'a}}.} \bibinfo{year}{2017}\natexlab{}.
\newblock \showarticletitle{Towards a Runtime Comparison of Natural and
  Artificial Evolution}.
\newblock \bibinfo{journal}{{\em Algorithmica\/}} \bibinfo{volume}{78},
  \bibinfo{number}{2} (\bibinfo{year}{2017}), \bibinfo{pages}{681--713}.
\newblock


\bibitem[\protect\citeauthoryear{Rajabi and Witt}{Rajabi and Witt}{2020}]%
        {CarstenSD2020}
\bibfield{author}{\bibinfo{person}{Amirhossein Rajabi} {and}
  \bibinfo{person}{Carsten Witt}.} \bibinfo{year}{2020}\natexlab{}.
\newblock \showarticletitle{Self-Adjusting Evolutionary Algorithms for
  Multimodal Optimization}. In \bibinfo{booktitle}{{\em Proceedings of the
  Genetic and Evolutionary Computation Conference}} {\em
  (\bibinfo{series}{GECCO '20})}. \bibinfo{publisher}{ACM},
  \bibinfo{pages}{1314–1322}.
\newblock


\bibitem[\protect\citeauthoryear{Rajabi and Witt}{Rajabi and Witt}{2021}]%
        {CarstenSD2021}
\bibfield{author}{\bibinfo{person}{Amirhossein Rajabi} {and}
  \bibinfo{person}{Carsten Witt}.} \bibinfo{year}{2021}\natexlab{}.
\newblock \showarticletitle{Stagnation Detection with Randomized Local Search}.
  In \bibinfo{booktitle}{{\em Evolutionary Computation in Combinatorial
  Optimization}}. \bibinfo{publisher}{Springer}, \bibinfo{pages}{152--168}.
\newblock


\bibitem[\protect\citeauthoryear{Rechenberg}{Rechenberg}{1973}]%
        {Rechenberg1973}
\bibfield{author}{\bibinfo{person}{Ingo Rechenberg}.}
  \bibinfo{year}{1973}\natexlab{}.
\newblock {\em \bibinfo{title}{Evolutionsstrategie}}.
\newblock \bibinfo{thesistype}{Ph.D. Dissertation}.
\newblock


\bibitem[\protect\citeauthoryear{Rowe and Sudholt}{Rowe and Sudholt}{2014}]%
        {Rowe2014}
\bibfield{author}{\bibinfo{person}{Jonathan~E. Rowe} {and}
  \bibinfo{person}{Dirk Sudholt}.} \bibinfo{year}{2014}\natexlab{}.
\newblock \showarticletitle{The Choice of the Offspring Population Size in the
  $(1,\lambda)$ Evolutionary Algorithm}.
\newblock \bibinfo{journal}{{\em Theoretical Computer Science\/}}
  \bibinfo{volume}{545} (\bibinfo{year}{2014}), \bibinfo{pages}{20--38}.
\newblock


\end{thebibliography}
\end{document}